\documentclass[oneside,11pt]{article}
\pdfoutput=1

\PassOptionsToPackage{pdftex}{graphicx}
\usepackage[pdftex]{graphicx}

\usepackage[ngerman,english]{babel}
\usepackage{geometry}
\usepackage[utf8]{inputenc}
\usepackage{amsmath, amsthm, amssymb}
\usepackage{amssymb}
\usepackage{graphicx}
\usepackage{caption} 
\usepackage{subcaption}
\usepackage[round,comma]{natbib}
\usepackage{dsfont}
\usepackage{algorithm}
\usepackage{algpseudocode}
\usepackage{float}

\def\banum#1\eanum{\begin{align}#1\end{align}} 
\newcommand{\RR}{\mathbb{R}}
 
\newcommand{\Nat}{\mathbb{N}}

\newcommand{\Xcal}{\mathcal{X}}

\newcommand{\Cc}{\mathcal{C}}
\newcommand{\Xc}{\mathcal{X}}
\newcommand{\1}{\mathbf{1}}
\renewcommand{\P}{\mathsf{P}}
\newcommand{\E}{\mathsf{E}}

\newtheorem{theorem}{Theorem} 
\newtheorem{lemma}[theorem]{Lemma} 
 
\newtheorem{remark}[theorem]{Remark}
\newtheorem{corollary}[theorem]{Corollary}

\DeclareMathOperator{\NN}{NN}

\begin{document}
\title{Comparison Based Nearest Neighbor Search}

\author{ Siavash Haghiri  \and Debarghya Ghoshdastidar \and Ulrike von Luxburg\\
Department of Computer Science, University of T{\"u}bingen, Germany\\ 
 \{haghiri,ghoshdas,luxburg\}@informatik.uni-tuebingen.de}

\date{}
\maketitle
\begin{abstract}

We consider machine learning in a comparison-based setting where we are given a set of points in a metric space, but we have no access to the actual distances between the points. Instead, we can only ask an oracle whether the distance between two points $i$ and $j$ is smaller than the distance between the points $i$ and $k$. We are concerned with data structures and algorithms to find nearest neighbors based on such comparisons. We focus on a simple yet effective algorithm that recursively splits the space by first selecting two random pivot points and then assigning all other points to the closer of the two (comparison tree). We prove that if the metric space satisfies certain expansion conditions, then with high probability the height of the comparison tree is logarithmic in the number of points, leading to efficient search performance. We also provide an upper bound for the failure  probability to return the true nearest neighbor. Experiments show that the comparison tree is competitive with algorithms that have access to the actual distance values, and needs less triplet comparisons than other competitors. 

\end{abstract}
\section{INTRODUCTION}

In many machine learning problems, data is given in the form of points and similarity or distance values between these points. Recently, comparison-based settings have become increasingly popular \citep{schultz2004learning,agarwal2007generalized,van2012stochastic,amid2015multiview,ukkonen2015crowdsourced,balcan2016Learning}. Here the assumption is that points come from some metric space $(\Xcal, d)$, but the metric $d$ is unknown. We only have indirect access to the metric in the form of {\bf triplet comparisons}: for a particular triplet of points $(x,y,z)$ we can ask whether 
\banum \label{triplet-comparison}
d(x,y) \leq d(x,z)
\eanum
is true or not. Such settings are popular in the crowd sourcing literature \citep{tamuz2011adaptively,heikinheimo2013crowd,ukkonen2015crowdsourced}. 

Assume that we are given a set of objects together with the answers to triplet comparisons for some (or even all) triplets of objects, and our task is to solve machine learning problems such as clustering or classification. There are two possible strategies we could pursue. The first is to construct an ordinal embedding of the objects into the Euclidean space $\RR^d$, that is an embedding such that the answers to the triplet comparisons are preserved \citep{agarwal2007generalized,van2012stochastic, terada2014local,kleindessner2014uniqueness,arias2015some,amid2015multiview,jain2016finite}. Subsequently, the machine learning problem can be solved by standard methods in Euclidean spaces. The second approach would be to use the triplet comparisons to find the $k$ nearest neighbors ($k$NN) of the points directly. We could then use $k$NN classification, clustering methods on the $k$NN graph, etc. In our paper, we focus on the second approach, and in particular on the question how many triplet comparisons are needed in order to find the $k$NNs of data points. 

Many algorithms for exact or approximate nearest neighbor search use data structures based on space partitioning. Most popular is the setting where points live in the Euclidean space. KD-Tree \citep{KdtreeBentley1975multidimensional}, PA-Tree \citep{mcnames2001fast}, Spill-Tree \citep{liu2004investigationSpill}, RP-Tree \citep{dasgupta2008random}, and MM-Tree \citep{ram2013space} are among the algorithms in Euclidean setting. In addition, the setting where data points only lie on a more abstract metric space has been explored considerably. Metric Skip List \citep{karger2002finding}, Navigating Net \citep{krauthgamer2004navigating}, and Cover-Tree \citep{beygelzimer2006cover} are some of the methods in this category. However, in nearly all these algorithms, we either need to know the vector representation of the points or the distance values between the points. Few exceptional methods exist that even work in the setting where we only have access to triplet comparisons \citep{goyal2008disorder,lifshits2009combinatorial,tschopp2011randomized,houle2015rank}. 

To our taste, the most appealing comparison-based data structure is what we call the comparison tree. The structure has been introduced as ``metric tree based on a generalized hyperplane decomposition" in \citet{uhlmann1991satisfying}. For the sake of simplicity we refer to it as the comparison tree. To partition the space into smaller subsets, two pivot points are picked uniformly at random from the given dataset. The space is then separated into two ``generalized half-spaces'', namely the two sets of points that are closer to the two respective pivots. This is done recursively until the resulting subsets become smaller than a given size. Once this tree structure has been constructed, the nearest neighbor search  proceeds by comparing the query point to the two pivot points. 
On an intuitive level, comparison trees are promising: (i) the splits seem to adapt to the geometry of the data, (ii) it seems that the splits are not extremely unbalanced, (iii) there is ``enough randomness" in the construction. Unfortunately, none of these intuitions has been proved or formally investigated yet. 

The first contribution of our paper is to analyze the performance of comparison trees in general metric spaces. Under certain assumptions on the expansion rates of the space, we prove that comparison trees are nicely balanced and their height is of the order $\Theta(\log n)$. This means that to construct a comparison tree, we only need of the order $n \log n$ triplet comparisons. Moreover, we can bound the probability that the nearest neighbor algorithm finds the correct nearest neighbor. Our second contribution consists of simulations that compare the behavior of comparison trees to standard data structures in Euclidean spaces and other comparison-based data structures on metric spaces. We find that the comparison tree performs surprisingly well even if compared to competitors that access vector representations in Euclidean spaces (KD-Tree, RP-Tree and PA-Tree), and favorably in comparison to one of the recent comparison-based algorithms proposed in \citet{houle2015rank}. 

\section{COMPARISON TREE}
\label{sec:algorithms}
Let $S$ be a set of $n$ points in some metric space $(\Xcal, d)$. {\bf
  To construct a comparison tree} on $S$, we proceed as follows (see Algorithm 1 below). The root of the tree $T.root$ consists of the whole set $S$, and each of the subsequent nodes represents a partition of the set $S$.  In each step of the tree construction, the elements of the current node are partitioned into two disjoint sets, which in turn are the root nodes for the left and right sub-trees denoted by $T.leftchild$ and $T.rightchild$. More concretely, to form a partition of the current node of the tree, we randomly choose two pivot elements among its current elements, denoted by $T.leftpivot$ and $T.rightpivot$. Then we group the remaining elements according to whether they are closer to the left or right pivot. Observe that this step does not require actual distance values, but just triplet comparisons. Then we recurse, until the current set of elements has at most size $n_0$ for some pre-specified $n_0$. 
\begin{algorithm}[h]
\caption{$CompTree(S,n_0)$: \textbf{Comparison tree construction}}
\label{alg:MT}
\begin{algorithmic}[1]
\Require $S\subseteq \Xc$, and maximum leaf size $n_0$
\Ensure Comparison tree $T$
\State $T.root \gets S$
\If{$|S| > n_0$}
\State Uniformly sample distinct points $x_1,x_2 \in S$
\State $S_1 \gets \{x \in S : d(x,x_1) \leq d(x,x_2)\}$
\State $T.leftpivot \gets x_1,$~~$T.rightpivot \gets x_2$
\State $T.leftchild \gets CompTree(S_1,n_0)$
\State $T.rightchild \gets CompTree(S \backslash S_1,n_0)$
\EndIf
\State \Return $T$
\end{algorithmic}
\end{algorithm}

The computational complexity of the tree construction is governed by the number of triplet comparisons required in the procedure, which depends on the height of tree. In the next section, we show that under certain growth assumptions on the metric, the comparison tree has height $h = O(\log n)$ with high probability.%

{\bf To use the comparison tree for (approximate) nearest neighbor search,} we employ the
obvious greedy procedure (sometimes called the defeatist search;  %
see Algorithm 2): starting at the root, we compare the query element $q$ to the current two pivot elements, and using a triplet comparison we decide whether to proceed in the left or right branch. When we reach a leaf, we carry out an exhaustive search among all its elements to determine the one that is closest to the given query. This step requires $n_0$ triplet comparisons. Overall, the nearest neighbor search in a comparison tree of height~$h$ requires at most $h + n_0$
triplet comparisons, which boils down to $O(\log n)$.%
\begin{algorithm}
\caption{$\NN(q,T)$: \textbf{Nearest neighbor search}}
\label{alg:NN}
\begin{algorithmic}[1]
\Require Comparison tree $T$, and query $q$.
\Ensure $\hat{x}_q=$approximate nearest neighbor of $q$ in $S$
\If{$T.leftchild \neq null$}
\If{$d(q,T.leftpivot) \leq d(q,T.rightpivot)$}
\State $\hat{x}_q \gets \NN(q,T.leftchild)$
\Else~~ $\hat{x}_q \gets \NN(q,T.rightchild)$
\EndIf
\Else~~ choose $\hat{x}_q$ s.t. $d(q,\hat{x}_q) \leq d(q,x)$ $\forall{x\in T.root}$
\EndIf
\State \Return $\hat{x}_q$
\end{algorithmic}
\end{algorithm}
\section{THEORETICAL ANALYSIS}
\label{sec:analysis}
In this section, we analyze the complexity and performance of the tree construction as well as the nearest neighbor search. We first provide a high probability bound on the height of the comparison tree, which in turn bounds the number of triplet comparisons required for both tree construction and nearest neighbor search. In addition, we derive an upper bound on the probability that the above approach fails to return the exact nearest neighbor of a given query. 
 
\subsection{Expansion Conditions}

Finding the nearest neighbor for a query point $q$ in a general metric space $(\Xc,d)$ can require up to $\Omega(n)$ comparisons in the worst case, using any data structure built on the given set $S$ \citep{beygelzimer2006cover}. Hence, most similarity search methods are analyzed under the natural assumption that the metric $d$ is growth-restricted \citep{karger2002finding,krauthgamer2004navigating}. Informally, such restrictions imply that the volume of a closed ball in the space does not increase drastically when its radius is increased by a certain factor. Various related notions are used to characterize the  growth rate of such metrics, for instance Assouad dimension \citep{assouad1979etude}, doubling dimension \citep{gupta2003bounded}, homogeneity \citep{luukkainen1998every}, and expansion rate \citep{karger2002finding} among others. 

The analysis of most tree based search methods in the Euclidean setting requires only a doubling property of the metric \citep{dasgupta2015randomized}. Such results hold for metric spaces with finite doubling dimension or Assouad dimension. When dealing with general metric spaces, it is more convenient to consider the expansion rate, which is an empirical variant of the doubling property defined for a given finite set $S\subseteq \Xc$. Typically, the analysis of  data dependent tree constructions requires
even stronger restrictions. 

In this work, we use a slightly weaker variant of the strong expansion rate condition used in \citet{ram2013space}. Intuitively, we need bounds on the expansion rate for all the finite point sets that can possibly occur in the non-leaf nodes of the comparison tree. 

Let $(\Xc,d)$ be a metric space, $S\subseteq \Xc$ and $n_0< |S|$.
We construct a collection $\Cc_S \subseteq 2^S$ as follows:
\begin{enumerate}\itemsep0em
 \item $S \in \Cc_S$,
 \item If $A \in 2^S$ and there exist $x,y\in S$ such that $d(z,x)\leq d(z,y)$ $\forall z\in A$,
 then $A \in \Cc_S$, $S\backslash A\in \Cc_S$, 
 \item If $A,B \in \Cc_S$, then $A \cap B\in \Cc_S$.
\end{enumerate}

We finally remove all $A \in \Cc_S$ with size $|A|\leq n_0$. Observe that $\Cc_S$ characterizes the collection of all possible non-leaf nodes of the tree\footnote{Technically, $\Cc_S$ is a subset of the algebra generated by the ``generalized half spaces'' of the induced space $(S,d)$.}. We define the {\bf strong expansion rate} of $S$ as the smallest $\tilde{c} \geq 1$ such that
\begin{equation}
|B(x,2 r) \cap A| \leq \tilde{c} |B(x,r) \cap A|
\label{eq:strongexpansion}
\end{equation}
for all $A\in \Cc_S$, $x\in A$ and $r>0$,
where $B(x,r)$ is the closed ball in $(\Xc,d)$
centered at $x$ with radius $r$.

Inequality \eqref{eq:strongexpansion} states that every
$A\in \Cc_S$ has an expansion rate at most $\tilde{c}$
similar to the definition in \citet{karger2002finding}.
This requirement for all $A\in \Cc_S$
is strong, but seems unavoidable due to the data dependent,
yet random, tree construction. 

 \subsection{Main Results}

The following theorem provides an upper bound on the height of the comparison tree.

\begin{theorem}[{\bf Height of a comparison tree}]
\label{thm:treeheight}
Consider a set $S$ of size $n$ in a metric space that satisfies the
strong expansion rate condition with constant $\tilde c$. Fix some
$n_0 \in \Nat$. Then for any $\epsilon>0$, 
with probability $1-\epsilon$, the comparison tree construction algorithm
 returns a tree with height smaller than
\begin{equation}
h^* = 3\log \left(\frac{e}{\epsilon}\right)
+ 96\tilde{c}^2 \log \left(\frac{n}{n_0}\right)\;.
\label{eq:treeheight}
\end{equation}
\end{theorem}

We prove the theorem later in the section. Theorem~\ref{thm:treeheight} implies that if the expansion rate $\tilde{c} = O(1)$, then the height of the randomly constructed tree tends to be bounded by $O(\log n)$. In particular, one can expect this to happen if the set of points is sampled from an ``evenly" spread distribution in a growth-restricted space $\Xc$. As a consequence of Theorem~\ref{thm:treeheight}, one can comment on the number of triplet comparisons required for tree construction and nearest neighbor search. We state this in the following corollary.

\begin{corollary}[{\bf Number of triplet comparisons}]
\label{cor:complexity}
For $\epsilon>0$, let $h^*$ be defined as in~\eqref{eq:treeheight}.
Then with probability $1-\epsilon$, the comparison tree construction algorithm
requires at most $nh^*$ triplet comparisons to construct the
comparison tree. 
Furthermore, for any $q\in \Xc$, with probability $1-\epsilon$,
nearest neighbor search algorithm uses at most $h^* + n_0$ triplet comparisons
to find an approximate nearest neighbor of $q$.  
\end{corollary}

The proof is a simple consequence of Theorem~\ref{thm:treeheight}.

Other applicable methods in our setting made various assumptions on the dataset, thus the upper bound on the required number of triplets is hardly comparable with them. If we neglect this fact and only compare the dependency to $n$, we can summarize the asymptotic bounds on the required number of comparisons in Table~\ref{tab:theorycomp}. We ignore all dependencies on the constants describing the geometric properties of the space (such as doubling, expansion, or disorder constants). 

\begin{table}[h]
    \centering
    \caption{Theoretical comparison with existing methods}
    \begin{tabular}{|c|c|c|c|} \hline
         Method & Construction  & Query \\ \hline\hline
          Comparison Tree & $n \log n$ & $\log n$  \\ \hline
          \citep{goyal2008disorder} & $n^2 \log n$ & $\log n$  \\ \hline
          \citep{lifshits2009combinatorial} & $n \log^2 n$ & $\log n$  \\ \hline
          \citep{tschopp2011randomized} & $n \log^2 n$ & $\log^2 n$  \\ \hline
          \citep{houle2015rank} & $n \log^3 n$ & $\log^3 n$  \\ \hline
    \end{tabular}
     \label{tab:theorycomp}
\end{table}

While the above discussion sheds light on the required number of triplet comparisons, it still leaves one wondering about the quality of the nearest neighbor obtained from the comparison tree. In the following result, we show that under certain conditions on the behavior of the metric in a neighborhood of a given query $q$, the search method succeeds in finding the true nearest neighbor of $q$. We use $x_q\in S$ to denote the true nearest neighbor of $q$, while $\hat{x}_q$ is the element returned by the nearest neighbor search. We write $B(x,r)$ and $B^{\circ}(x,r)$ to denote closed and open balls, respectively.

\begin{theorem}[{\bf Exact nearest neighbor}]
\label{thm:error}
Given $S\subseteq \Xc$ and $q\in \Xc$.
If there exist constants $C>0$ and $\alpha\in(0,1]$ such that 
for every  $A\in \Cc_S$ containing $x_q$, and for all $x\in A\backslash \{x_q\}$,
\begin{align}
&\left| B\big(q,d(q,x)+2d(q,x_q)\big) \cap A\right|
\nonumber \\
&\qquad\qquad\leq \left| B^{\circ}\big(q,d(q,x)\big) \cap A\right| + C|A|^{1-\alpha} \;,
\label{eq:querycondn}
\end{align}
then
\begin{equation}
\P\left(\hat{x}_q \neq x_q\right) 
\leq \frac{360 C \tilde{c}^2}{\alpha} n_0^{-\alpha}\;,
\label{eq:errorbound}
\end{equation}
where the probability is with respect to the random construction of the tree.
\end{theorem}
The condition on $q$ implies that there are not many points that have the same distance to $q$ as the nearest neighbor. Under the local restrictions defined in \eqref{eq:querycondn} on the query $q$, the error bound~\eqref{eq:errorbound} states that one can achieve an arbitrarily small error probability if $n_0$ is chosen large enough, depending on the strong expansion rate $\tilde{c}$.

\begin{remark}
The assumption in~\eqref{eq:querycondn} can also be substituted by alternative conditions.
For instance, the error bound in~\eqref{eq:errorbound} holds (up to constants) if there exists $D>1$ such that
\begin{equation}
|B(q,\lambda r) \cap A| \leq \lambda^{D} |B(q,r) \cap A|
\label{eq:strongstrongexpansion}
\end{equation}
for all $A\in \Cc_S$, $\lambda>1$, and $r>d(q,x_q)$. 
One can see the inherent resemblance of~\eqref{eq:strongstrongexpansion} to the notion of strong expansion rate~\eqref{eq:strongexpansion}.
\end{remark}

We remark again on the required conditions. The notion of strong expansion rate, though used in \citet{ram2013space}, is stronger than standard conditions used in many works \citep{dasgupta2015randomized,karger2002finding}. Our main reason for resorting to this notion is because of the data dependent random splits used in comparison tree construction. While projection-based methods also use random hyperplanes for splitting each node, such hyperplanes are independent of the given set, making the analysis simpler \citep{dasgupta2015randomized}. On the other hand, prior works in non-Euclidean setting construct data structures that naturally adhere to the structure of the metric balls \citep{karger2002finding}. Unlike both these works, in the present setting, one cannot guarantee that a condition defined on the whole set will also hold for each of the partitions obtained during splits. Hence, the condition of strong expansion rate has been used in our analysis. The additional assumption on $q$ seems essential since one can construct trivial examples where the nearest neighbor search is quite likely to fail.

\subsubsection{Proof of Theorem~\ref{thm:treeheight}}
We now prove Theorem~\ref{thm:treeheight} using two lemmas.

\begin{lemma}[{\bf Probability of unbalanced split}]
\label{lem:fairsplit}
For any $A\in \Cc_S$ and $\delta\in(0,1)$, the probability
that the random split in the comparison tree construction algorithm creates a child
of $A$ with less than $\delta|A|$ elements is at most $4\tilde{c}^2\delta$.
\end{lemma}

Thus, each split in the tree is reasonably balanced,
and hence, it is likely that the size of the nodes decays rapidly
with their depth. This fact is formalized below.

\begin{lemma}[{\bf Maximum node size at depth $h$}]
\label{lem:nodesize}
Let $A\in \Cc_S$ be a node at depth $h$ of the tree. 
If $4\tilde{c}^2\delta\leq1$, then
the probability that $A$ has more than $m$ elements is at most
$(8\tilde{c}^2\delta)^{h-1}\left({{n}/{m}}\right)^{
{({1}/{\delta})} \log({{1}/{4\tilde{c}^2\delta})}}$.
\end{lemma}

We finish the proof of Theorem~\ref{thm:treeheight} by observing that
the height of the tree is greater than $h$ only if there is a node at 
depth $h$ of size greater than $n_0$. By taking a union over all possible 
$2^h$ nodes at depth $h$, one can see that the probability of this event is
at most $2(16\tilde{c}^2\delta)^{h-1}\left({{n}/{n_0}}\right)^{
({{1}/{\delta}}) \log({{1}/{4\tilde{c}^2\delta})}}$. This probability is less than $\epsilon$
if we fix $\delta = {{1}/{32\tilde{c}^2}}$, and 
\begin{displaymath}
h = \left( 1 + 2\log \left(\frac{e}{\epsilon}\right)
+ 96\tilde{c}^2 \log \left(\frac{n}{n_0}\right) \right)
\leq h^* \;.
\end{displaymath}

We now prove the above two lemmas.
\begin{proof}[Proof of Lemma~\ref{lem:fairsplit}]
Let $|A|=m$, and $\1(\cdot)$ denote the indicator function. Then
the probability of splitting $A$ to create a child 
of size smaller than $\delta m$ is at most
\begin{align*}
&{\textstyle\frac{1}{m(m-1)}}
\hspace{-4mm}
\sum_{\substack{x_1\in A\\x_2\in A\backslash\{x_1\}} }
\hspace{-4mm}
\1\Big( \{ | \{x\in A: d(x,x_2) \leq d(x,x_1)\}| \leq \delta m\}
\nonumber
\\& \hspace{17mm}
\cup \{ | \{x\in A: d(x,x_1) < d(x,x_2)\}| \leq \delta m\} \Big)
\nonumber
\\&\leq
{\textstyle\frac{1}{\binom{m}{2}}}
\hspace{-4mm}
\sum_{\substack{x_1\in A\\x_2\in A\backslash\{x_1\}} }
\hspace{-4mm}
\1\Big( \big| \{x\in A: d(x,x_1) < d(x,x_2)\} \big| \leq \delta m \Big).
\end{align*} 
Note that the set $\{x\in A: d(x,x_1) < d(x,x_2)\}$ contains the  set
$B\left(x_1, \frac{1}{4} d(x_1,x_2)\right)\cap A$, where $B(\cdot,\cdot)$ is the closed ball.
Thus, one may bound the above probability by the fraction of $x_1,x_2$
pairs for which this ball contains less than $\delta m$ elements.
Moreover, using the condition of strong expansion rate~\eqref{eq:strongexpansion},
one has
\begin{align*}
\left| B\big(x_1,d(x_1,x_2)\big) \cap A \right|
&\leq \tilde{c}^2
\left| B\left(x_1, \textstyle\frac{1}{4} d(x_1,x_2)\right) \cap A\right| \;.
\end{align*}
Thus, one may only count the $x_1,x_2$ pairs for which
$B(x_1,d(x_1,x_2)) \cap A$ contains at most $\tilde{c}^2\delta m$
elements. Now, for every $x_1$, if one sorts $x_2$ in the increasing order 
of $d(x_1,x_2)$, then the indicator is true only for the 
first $\tilde{c}^2 \delta m$ of $x_2$'s.
Thus, the probability of an unbalanced split 
is at most ${{2\tilde{c}^2\delta m^2}/{m(m-1)}}\leq 4\tilde{c}^2\delta$.
\end{proof}
\begin{proof}[Proof of Lemma~\ref{lem:nodesize}]
We denote the path from the root of the tree to $A$ by
$S = A_1 \supset A_2 \supset \ldots \supset A_{h-1} \supset A_h = A$.
Let $A'_j$ denote the sibling of $A_j$ for $j\geq 2$.
By the Markov inequality, one can write for any $t>0$,
\begin{align*}
\P(|A| > m) 
& \leq \left(\frac{n}{m}\right)^t \E\left[ \frac{|A_h|^t}{|A_1|^t}\right] 
\\&= \left(\frac{n}{m}\right)^t \E\left[ \frac{|A_{h-1}|^t}{|A_1|^t} \E\left[ 
\left. \frac{|A_{h}|^t}{|A_{h-1}|^t} \right| A_{h-1} \right] \right] .
\end{align*}
One can bound the inner conditional expectation as
\begin{align*}
&\E\left[\left. \frac{|A_{h}|^t}{|A_{h-1}|^t} \right| A_{h-1} \right]
\\&= \E\left[\left. \frac{|A_{h}|^t}{|A_{h-1}|^t} \1\big(|A_h|>(1-\delta)|A_{h-1}|\big) \right| A_{h-1} \right]
\\&\qquad + \E\left[\left. \frac{|A_{h}|^t}{|A_{h-1}|^t} \1\big(|A_h|\leq(1-\delta)|A_{h-1}|\big) \right| A_{h-1} \right]
\\&\leq  \P\big(|A'_h|\leq \delta|A_{h-1}|\big)
+ (1-\delta)^t \;,
\end{align*} 
where the inequality follows by replacing the ratio by 1 in the first expectation,
and $|A_h|$ by its upper bound in the second one. 
Due to Lemma~\ref{lem:fairsplit}, one can see that this bound is at most
$\left(4\tilde{c}^2 \delta + (1-\delta)^t\right)$. 
For $t = ({{1}/{\delta}})\log({{1}/{4\tilde{c}^2\delta}})$, one can use the fact that
$({{1}/{\delta}})\log({{1}/{(1-\delta)}}) \geq 1$ to 
show that the above expectation is at most $8\tilde{c}^2\delta$.

Subsequently, we use the same technique of conditioning with every $A_j$,
$j=2,\ldots,h-2$ to obtain
\begin{displaymath}
\P(|A|>m) \leq \left(\frac{n}{m}\right)^{\frac{1}{\delta}\log(\frac{1}{4\tilde{c}^2\delta})} (8\tilde{c}^2\delta)^{h-1} .
\end{displaymath}\vspace{-.7cm}
\end{proof}
\subsubsection{Proof of Theorem~\ref{thm:error}}
The nearest neighbor search for a query point~$q$ is done by traversing the tree from the root to one of the leaves. 
Let us denote the visited path by~$S = A_1 \supset A_2 \supset \ldots \supset A_{k-1} \supset A_k$,
where $A_k$ is the leaf node containing $\hat{x}_q$.
We assume that $q\notin S$, as otherwise the nearest neighbor search algorithm returns the query. 
By simple reasoning, it follows that $\hat{x}_q\neq x_q$ only if $x_q\notin A_k$,
which happens if there is $l\in\{1,2,\ldots,k-1\}$ such that $x_q\in A_l \backslash A_{l+1}$. 
Hence,
\begin{align}
&\P( \hat{x}_q \neq x_q)
=\P \left(\bigcup_{l=1}^{k-1} \{x_q \in A_l, x_q \notin A_{l+1} \}\right)
\nonumber
\\&\leq \sum_{l=1}^{k-1} \P (x_q \notin A_{l+1} | x_q \in A_l) \P( x_q \in A_l) 
\nonumber
\\&\leq \sum_{l=1}^{k-1} \P \left(x_q \notin A_{l+1} \left| x_q \in A_l,|A_l| \geq m_l \right.\right) \P\left(|A_l| \geq  m_l \right) +
\nonumber
\\&\P \left(x_q \notin A_{l+1} \left| x_q \in A_l,|A_l| <m_l \right.\right) \P\left(|A_l| <  m_l \right)
\label{eq:errorboundproof}
\end{align}
where $m_l={{n_0}/{(1-\gamma)^{k-1-l}}}$ for some $\gamma\in(0,1)$.
The first inequality is due to union bound, while the second one uses $\P( x_q \in A_l)\leq 1$ and
further decomposes based on $|A_l|$.

\begin{lemma}[{\bf Probability of missing nearest neighbor in one branch}]
\label{lem:losingNN}
Under the condition on $q$ stated in Theorem~\ref{thm:error},
for any $l=1,2,\ldots,k-1$,
\begin{displaymath}
\P (x_q \notin A_{l+1} | A_l, x_q \in A_l)  \leq C|A_l|^{-\alpha} \;.
\end{displaymath}
\end{lemma}

Note that for the two conditional probabilities in~\eqref{eq:errorboundproof}, $|A_l|$ is at least
$m_l$ and $n_0$, respectively. Using Lemma~\ref{lem:losingNN},
one can bound these probabilities.
To obtain a bound on $\P\left(|A_l| <  m_l \right)$,
we follow Lemma~\ref{lem:nodesize}.

Observe that Lemma~\ref{lem:nodesize} implies that after repeated splits, 
it is less likely that the ratio of the final node to the root node will be large. In the present context, 
we know that $|A_{k-1}| \geq n_0$. Thus, for any $l<k-1$,  
the bound in Lemma~\ref{lem:nodesize} can be used to argue that
${{|A_{k-1}|}/{|A_l|}}$ cannot be large. Formally,
\begin{align*}
\P\left(|A_l| <  m_l \right) &=
\P\left(\frac{|A_{k-1}|}{|A_l|} >  \frac{n_0}{m_l} \right)
\\&\leq (8\tilde{c}^2\delta)^{k-1-l}\left(\frac{1}{(1-\gamma)^{k-1-l}}\right)^{
\frac{1}{\delta} \log(\frac{1}{4\tilde{c}^2\delta})}
\\&= \left(\frac{0.25}{(1-\gamma)^{96\tilde{c}^2\log 2}}\right)^{k-1-l},
\end{align*}
using $\delta = {{1}/{32\tilde{c}^2}}$ as in Theorem~\ref{thm:treeheight}.
Substituting the above bound in~\eqref{eq:errorboundproof}
and using Lemma~\ref{lem:losingNN}, we have
\begin{align*}
\P(\hat{x}_q\neq x_q)
\leq &\sum_{l=1}^{k-1} Cn_0^{-\alpha} (1-\gamma)^{\alpha(k-1-l)}
\\&+ Cn_0^{-\alpha}\left(\frac{0.25}{(1-\gamma)^{96\tilde{c}^2\log 2}}\right)^{k-1-l}.
\end{align*}
Choosing $\gamma = 1 - (0.25)^{1/(\alpha+96\tilde{c}^2\log 2)}$, one can see that the second term is smaller
than the first, and hence, 
\begin{align*}
\P(\hat{x}_q\neq x_q)
&\leq \frac{2Cn_0^{-\alpha}}{1 - (1-\gamma)^{\alpha}}
\leq \frac{2Cn_0^{-\alpha}}{1-(0.25)^{\alpha/67.5\tilde{c}^2}} \;.
\end{align*}
From above, we obtain the bound in~\eqref{eq:errorbound} by using the relation 
$1-a \leq b(1-a^{1/b})$, which holds for any $a\in(0,1)$ and $b>1$.
This proves Theorem~\ref{thm:error}.

We end the section with the proof of Lemma~\ref{lem:losingNN}.
\begin{proof}[Proof of Lemma~\ref{lem:losingNN}]
Let $|A_l| = m$, and let us order the elements such that $d(q,x_i)\leq d(q,x_j)$ for $i<j$.
Since $x_q\in A_l$, we have $x_1 = x_q$. Note that if $x_q\notin A_{l+1}$, then it is certainly not a pivot element.
Moreover, if $x_i,x_j\in A_l$ are the pivot elements for $i<j$, then $x_q\notin A_{l+1}$ implies $d(x_q,x_i)\geq d(x_q,x_j)$.
The inequality can even be strict depending on which is chosen as the left pivot. Hence, we have
\begin{align*}
\P (x_q &\notin A_{l+1} | A_l, x_q \in A_l)  
\\&\leq
\frac{1}{m(m-1)} \sum_{2\leq i<j\leq m} \1\big(d(x_q,x_i)\geq d(x_q,x_j)\big)\;.
\end{align*}
By the triangle inequality,
\begin{align*}
&d(x_q,x_i) \leq d(q,x_i) + d(q,x_q)  \text{~~and}
\\&d(x_q,x_j) \geq d(q,x_j) - d(q,x_q) \;.
\end{align*}
Hence, one may count the pairs $i<j$ for which
\begin{displaymath}
d(q,x_j) \leq d(q,x_i) + 2d(q,x_q)\;.
\end{displaymath}
As a consequence, we can write
\begin{align*}
&\P (x_q \notin A_{l+1} | A_l, x_q \in A_l)  
\leq  \frac{1}{m(m-1)}\cdot
\\&\sum_{i=2}^m \big| \{x\in A: d(q,x_i) \leq d(q,x) \leq d(q,x_i)\text{+}2d(q,x_q)\}\big|,
\end{align*}
where each term in the sum is at most $Cm^{1-\alpha}$ due to the assumption on $q$. Thus, the claim of the lemma is true.
\end{proof}

\begin{figure*}\hspace{-.35cm}
    \centering
    \begin{subfigure}[b]{0.238\textwidth}
        \centering
        \includegraphics[width=\textwidth]{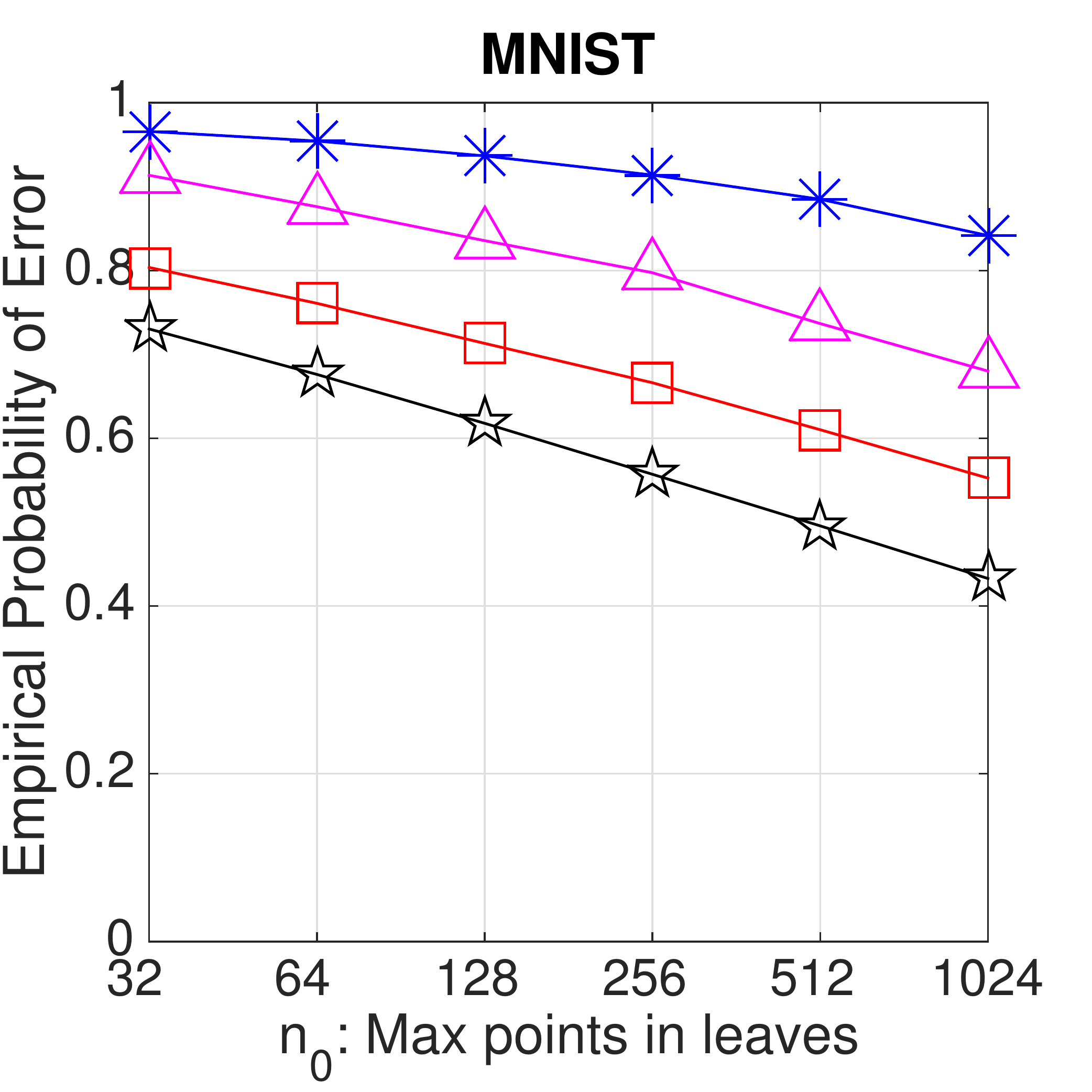}
        \vspace{-.5cm}\caption{MNIST}
        \label{fig:mnist}
    \end{subfigure} \hspace{-.35cm}
    \begin{subfigure}[b]{0.238\textwidth}
        \centering
        \includegraphics[width=\textwidth]{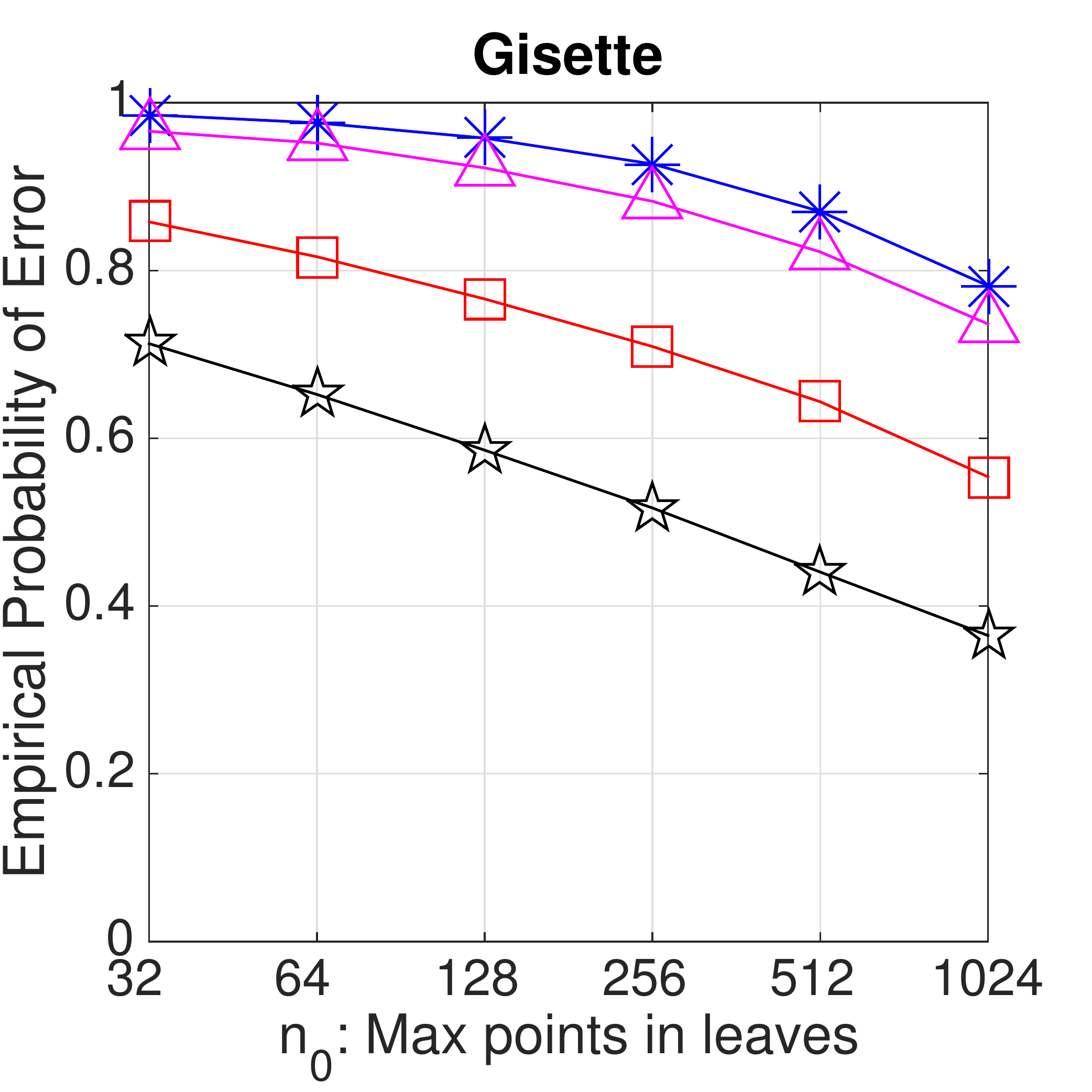}
        \vspace{-.5cm}\caption{Gisette}
        \label{fig:gisette}
    \end{subfigure}\hspace{-.3cm}
    \begin{subfigure}[b]{0.238\textwidth}
        \centering
        \includegraphics[width=\textwidth]{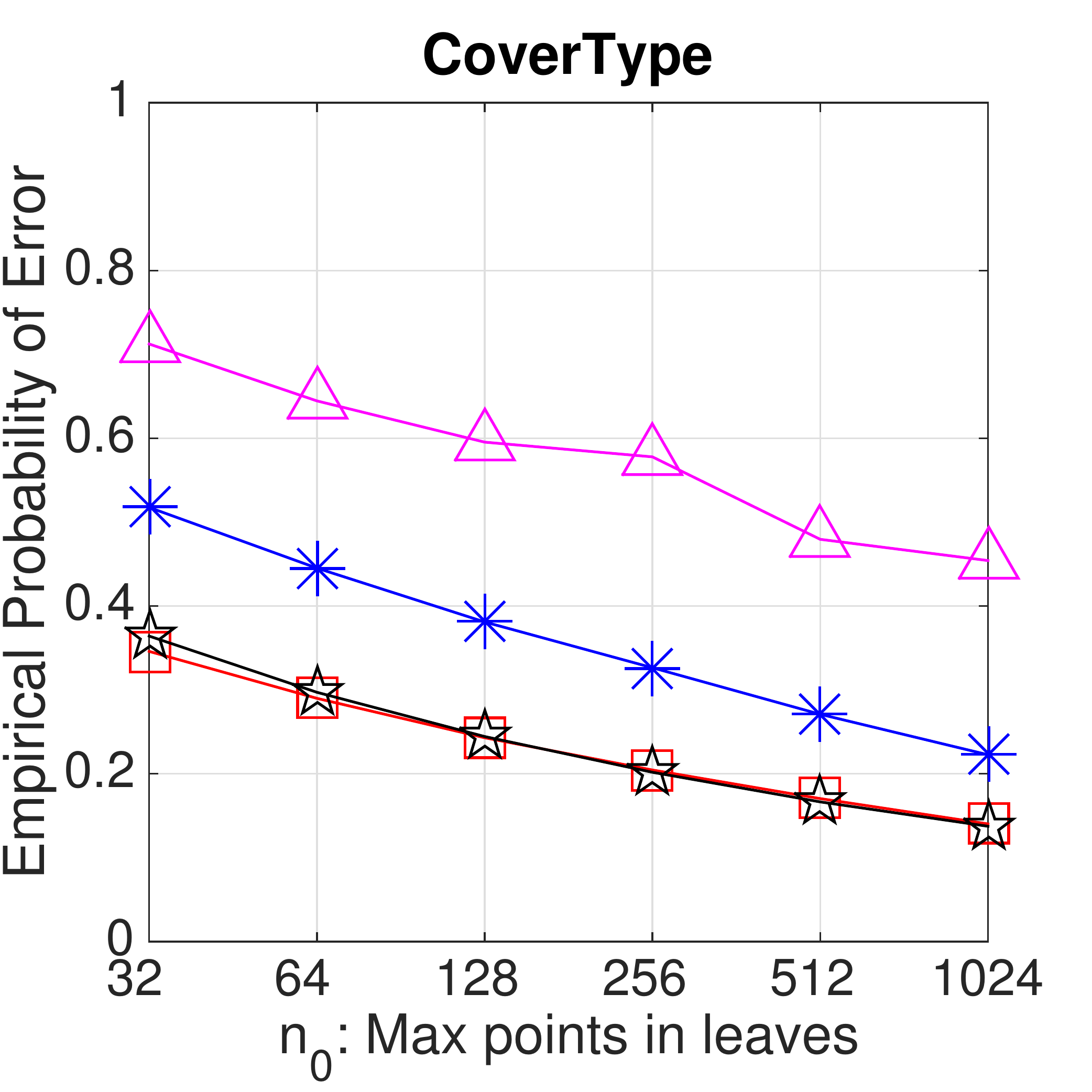}
        \vspace{-.5cm}\caption{CoverType}
        \label{fig:covtype}
    \end{subfigure}\hspace{-.3cm}
    \begin{subfigure}[b]{0.338\textwidth}
        \centering
        \includegraphics[width=\textwidth]{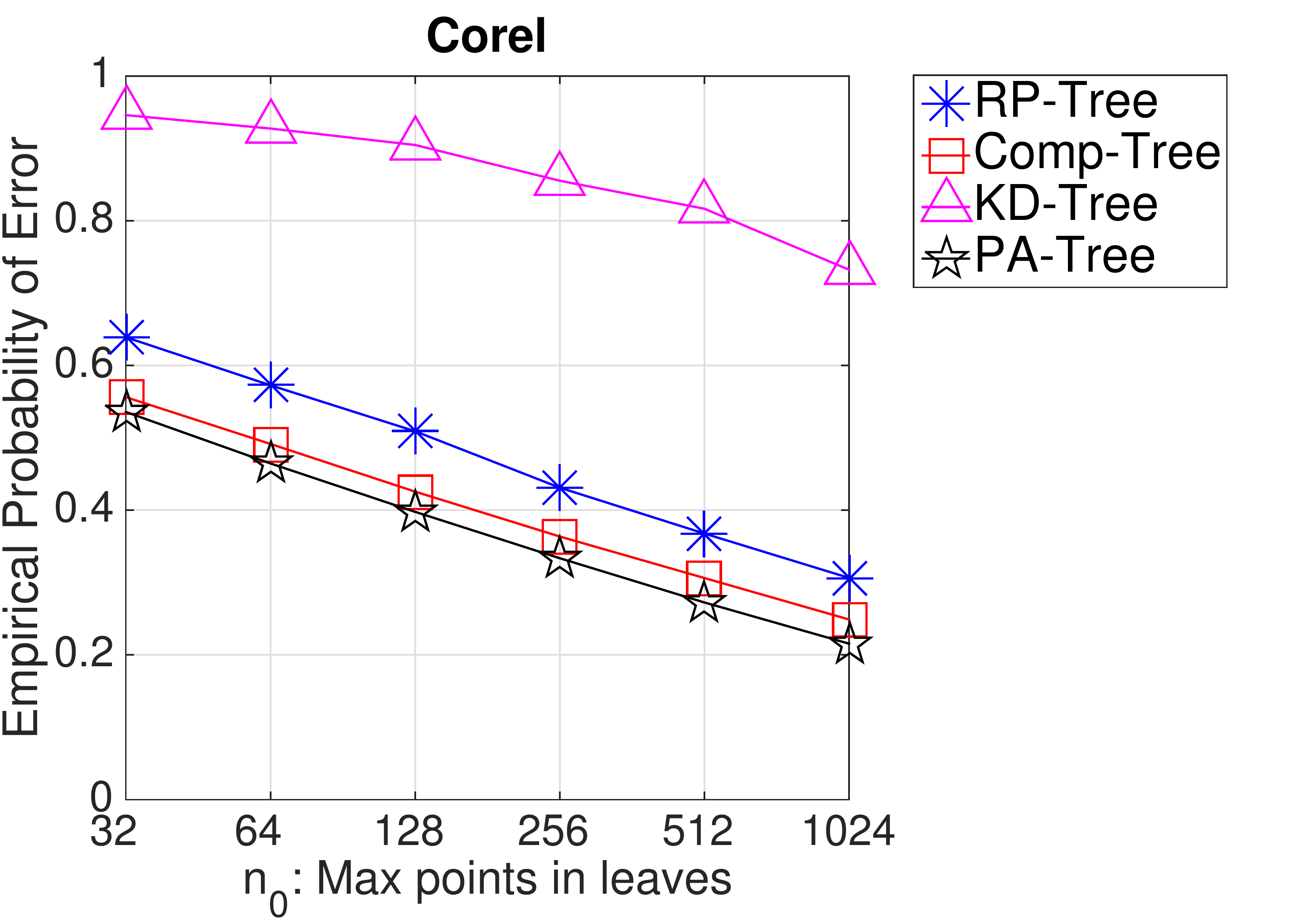}
        \vspace{-.5cm}\caption{Corel}
        \label{fig:corel}
    \end{subfigure}\hspace{-.45cm}
    \caption{Performance comparison of the comparison tree versus other binary space partitioning trees with respect to the maximum leaf size, $n_0$.}
    \label{fig:BSPComparisons}
\end{figure*}
\section{EXPERIMENTS}

\subsection{Euclidean Setting}

In this section we compare the performance of comparison trees to standard space partitioning trees in Euclidean spaces. Note that the latter have access to the vector representation of the points (and thus also to all pairwise distances), whereas the comparison tree only has access to triplet comparisons. Thus, the purpose of this comparison cannot be to show that the comparison tree ``outperforms'' the other ones, but to examine whether it is much worse or not. There are numerous tree constructions in Euclidean spaces. Based on the results in \citet{ram2013space} we decided to compare with KD-Tree \citep{KdtreeBentley1975multidimensional}, RP-Tree \citep{dasgupta2008random} and PA-Tree %
\citep{mcnames2001fast}. 

A description of the datasets is presented in Table~\ref{tab:datasets}. MNIST is a dataset of hand-written digits \citep{lecun1998gradient}. Gisette, CoverType and Chess (King-Rook vs. King) are from the UCI repository \citep{Lichman:2013}. Corel is a subset of histograms as it is used in \citet{liu2004investigationSpill}. CoAuth is the collaboration network of Arxiv High Energy Physics from \citet{floridaSparse}. We used the largest connected component of the graph and the shortest path as metric. MSC (Boeing/msc10848) is a similar but weighted graph from \citet{floridaSparse}\footnote{There are negative edge weights in the graph, however we used absolute values of edge weights to have a metric by using shortest path distances.}. %

\begin{table}[ht]
    \centering
    \caption{Description of datasets}
    \begin{tabular}{|c|c|c|c|} \hline
         Dataset & Size & Dimension & Distance \\ \hline\hline
         MNIST & 70000 & 784 & Euclidean \\ \hline
         Gisette & 12500 & 5000 & Euclidean \\ \hline
         CoverType & 50000 & 53 & Euclidean \\ \hline
         Corel & 19787 & 44 & Euclidean \\ \hline
         Chess & 28056 & 6 & Mismatch \\ \hline
         CoAuth & 11204 & - & Shortest Path \\ \hline
         MSC & 10848 & - & Shortest Path \\ \hline
    \end{tabular}
     \label{tab:datasets}
\end{table}

We assess the performance of the nearest neighbor search by the leave-one-out method. As the performance measure, we report the empirical probability of missing the nearest neighbor: $({{1}/{\vert S \vert}}) \sum_{q \in S} \mathds{1}\left( \NN_{alg}(q) \neq \NN(q)\right)$. Here $S$ is the whole dataset, $\NN_{alg}$ denotes the result of nearest neighbor search by the algorithm while the true nearest neighbor is $\NN(q)$. 

Figure~\ref{fig:BSPComparisons} shows the performance of the comparison tree versus other methods in the Euclidean space. The comparison tree has less error compared to RP-Tree and KD-Tree, and has slightly worse performance comparing with the PA-Tree. However, the differences are not huge, and we find the behavior of the comparison tree quite satisfactory, given that it receives much less input information than the other methods. 

\subsection{Comparison-Based Setting}

Among the few comparison-based methods cited in the introduction, many are not practical or have already been shown to perform sub-optimally. The most promising competitor to our method is the Rank Cover Tree (RCT) \citep{houle2015rank}. As the original implementation of the authors was not available, we implemented the method ourselves in Matlab. 

We have two objectives when comparing the two comparison-based trees: the number of required triplet questions, and the accuracy they achieve in the nearest neighbor search. While the latter is easy to compare, the former is more of a challenge. It is impossible to construct  an RCT with the same number of triplets that the comparison tree requires in construction phase, since the RCT needs orders of magnitude more triplets in construction. Thus, we decided to construct both trees in such a way that the number of triplet comparisons in the \textit{query phase} is matched. We then compare the search performance, but also the number of triplets in the tree construction phase. For the RCT, the performance and the number of comparisons can be balanced by adjusting the coverage parameter $\omega$, see \citet{houle2015rank} Section 4. For comparison trees, $n_0$ plays a similar role. By varying these two parameters we match the number of comparisons in the query phase.%

We randomly choose 1000 data points in each experiment as test set for the query phase and the rest of the dataset for the tree construction. The empirical error defined in previous section is not well-defined for some of datasets in this section. In CoAuth and Chess dataset, many points have more than one nearest neighbor. Thus, we report the average relative distance error defined as $({{1}/{\vert S \vert}})\sum_{q\in S}\left(\frac{d_{alg}}{d_{\NN}} -1\right)$\citep{liu2004investigationSpill}. Here $d_{alg}$ denotes the distance of query to the predicted nearest neighbor by the algorithm and $d_{\NN}$ denotes the distance of the query to the true nearest neighbor.
\begin{figure*}[ht]
    \centering
    \begin{subfigure}{0.23\textwidth}
        \centering
        \includegraphics[width=\textwidth]{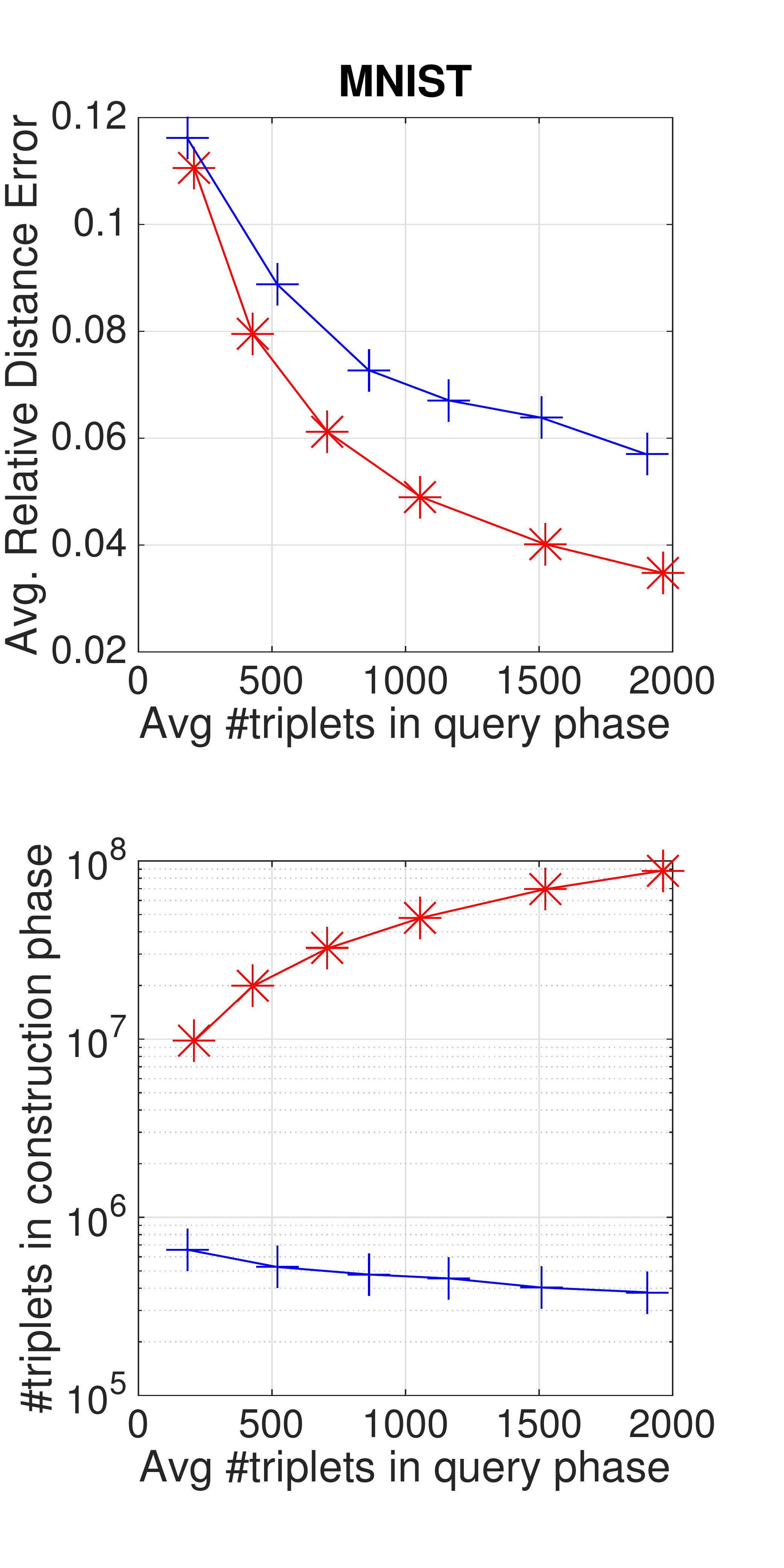}
        \vspace{-.8cm}\caption{MNIST}
        \label{fig:a}
    \end{subfigure} \hspace{-.4cm}
    \begin{subfigure}{0.23\textwidth}
        \centering
        \includegraphics[width=\textwidth]{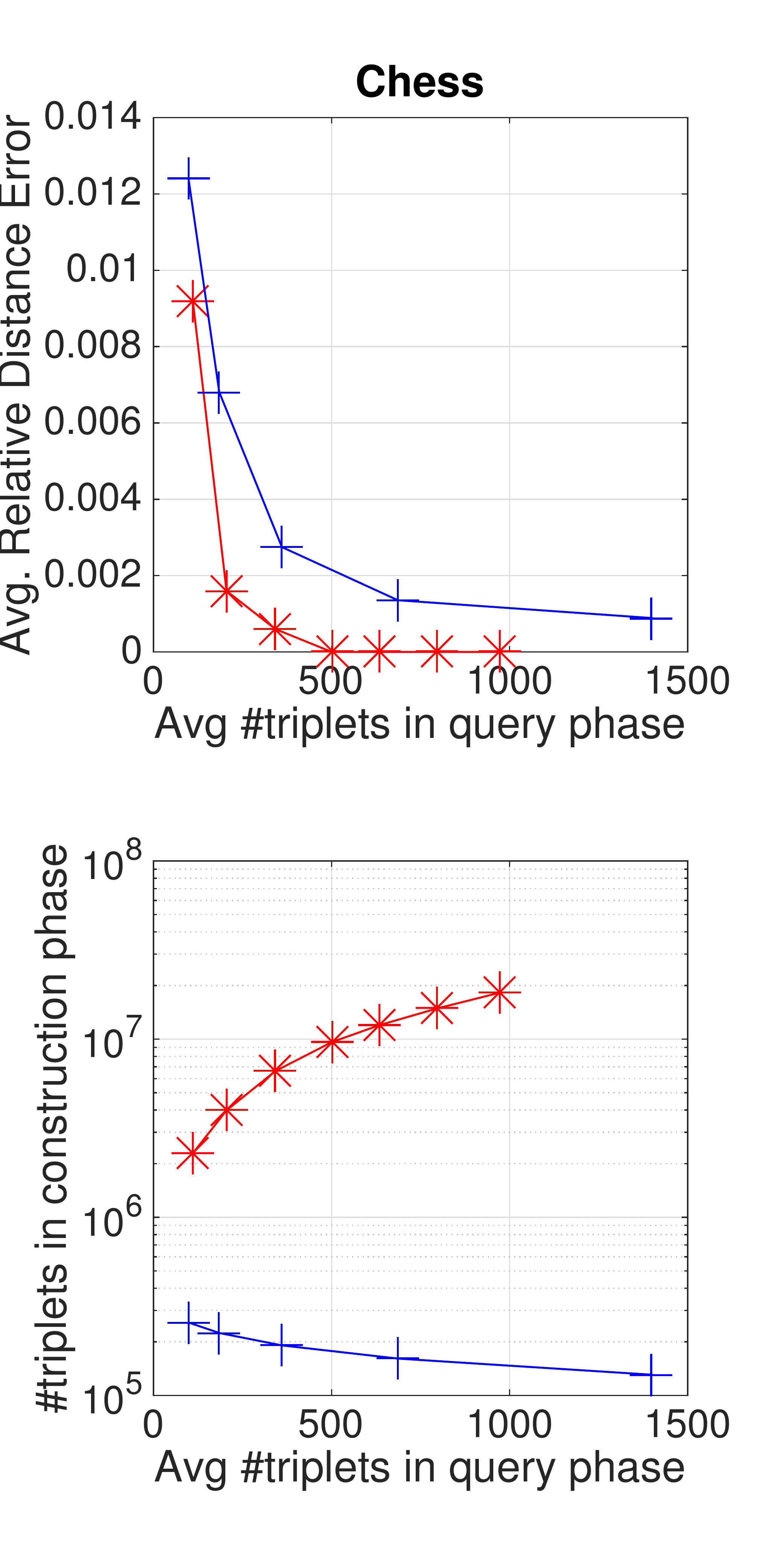}
        \vspace{-.8cm}\caption{Chess}
        \label{fig:e}
    \end{subfigure}\hspace{-.35cm}
    \begin{subfigure}{0.23\textwidth}
        \centering
        \includegraphics[width=\textwidth]{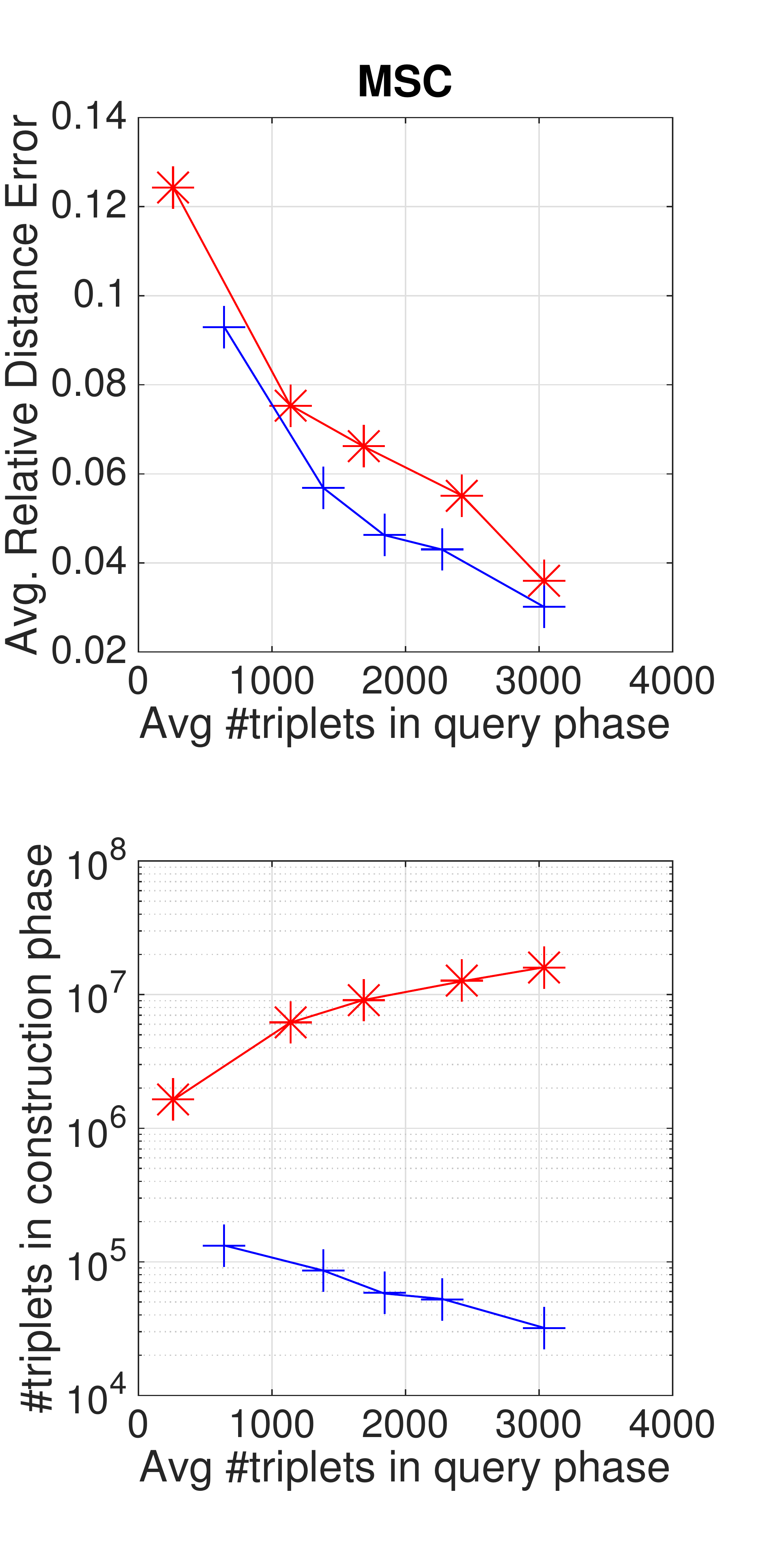}
        \vspace{-.8cm}\caption{MSC}
        \label{fig:f}
    \end{subfigure}\hspace{-.4cm}
    \begin{subfigure}{0.33\textwidth}
        \centering
        \includegraphics[width=\textwidth]{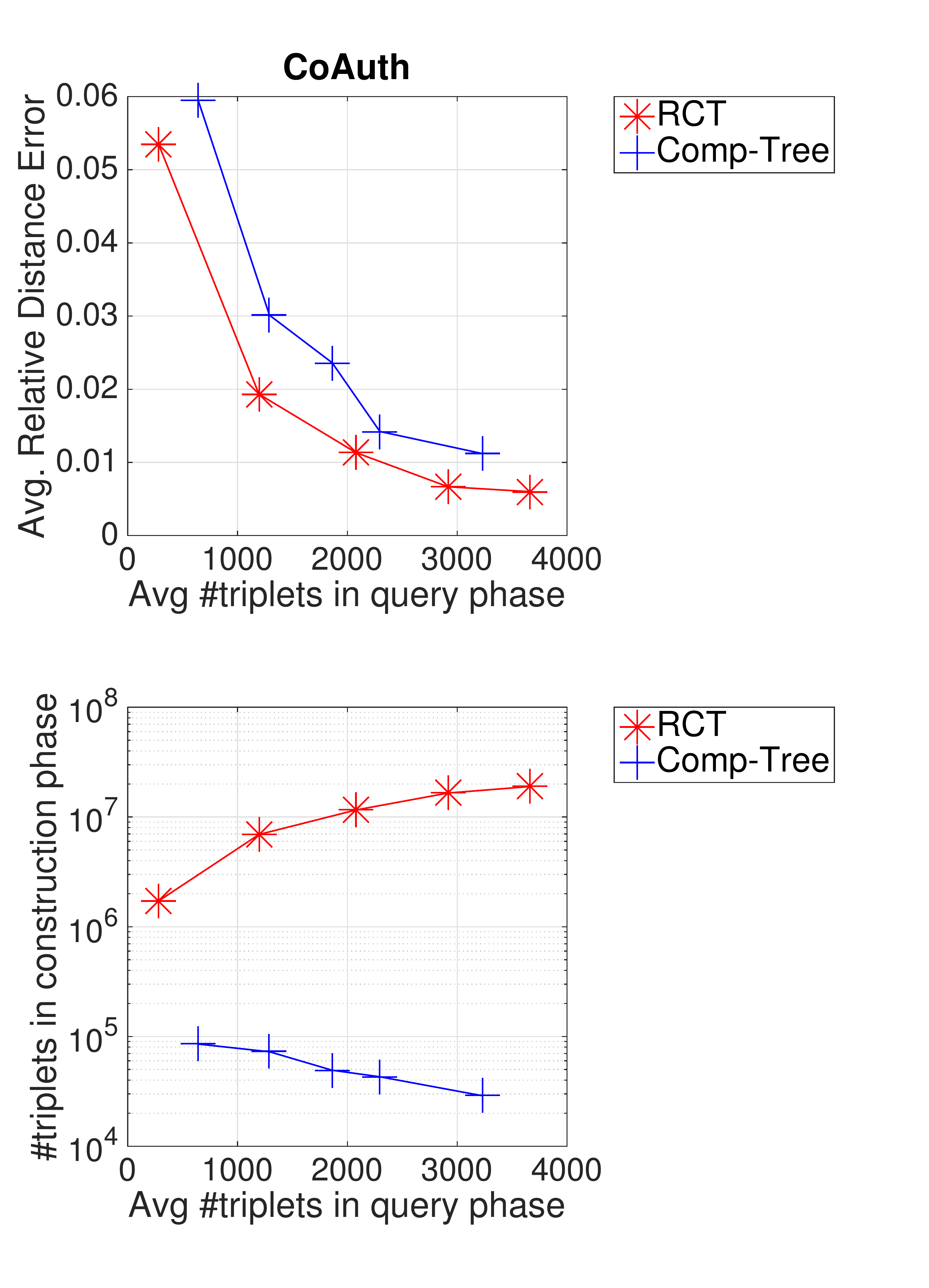}
        \vspace{-.8cm}\caption{CoAuth}
        \label{fig:g}
    \end{subfigure}\hspace{-.5cm}
        \vspace{-.2cm}\caption{Comparing the performance of the comparison tree versus the RCT with various parameters in the construction phase. The top row corresponds to the average relative distance error with respect to the average number of triplets that each method used in the query phase. The bottom row shows the number of triplets used in the construction phase for both methods. Note that x-axis is the same in both rows.}
    \label{fig:RCTMT}
\end{figure*}

Figure~\ref{fig:RCTMT} shows the performance of the comparison tree compared to the RCT on four datasets from Table~\ref{tab:datasets}. The results on the remaining Eulidean datasets are very similar to MNIST, hence we do not present them. We consider different parameter settings and match the average number of triplets used in the query phase. In terms of the relative distance errors, the RCT works slightly better in datasets with low intrinsic dimension, specially when we are provided with more triplets in query phase. However, as the bottom row in Figure~\ref{fig:RCTMT} shows, to achieve this performance the RCT needs orders of magnitude more triplet comparisons in the tree construction phase. Therefore, if answering triplet comparisons is expensive, then the comparison tree clearly is a good alternative to the RCT.%

\subsection{Expansion Rate Approximation}
In our theoretical analysis, we used the strong expansion condition defined in Equation~\eqref{eq:strongexpansion}. It is an obvious question to find out how strong these conditions really are and what the corresponding constants in our datasets would be. To this end, we provide a method to estimate the expansion rates for our datasets. We fix a dataset, for each point $x$ we look for the smallest $\tilde{c}$ such that Equation~\eqref{eq:strongexpansion} holds for that particular point. We find the smallest value with respect to various radii~$r$. In this way we estimate an empirical pointwise $\tilde{c}(x)$ value for each point. Since the definition depends on the number of points in the dataset, we randomly choose 10000 points from each dataset for these experiments. The distribution of empirical expansion rates~$\tilde{c}(x)$ are plotted by box-and-whisker plots in Figure~\ref{fig:ExpRatesAll}. 

\begin{figure}[ht]
    \centering
    \begin{subfigure}[b]{0.19\textwidth}
        \centering
        \includegraphics[width=\textwidth]{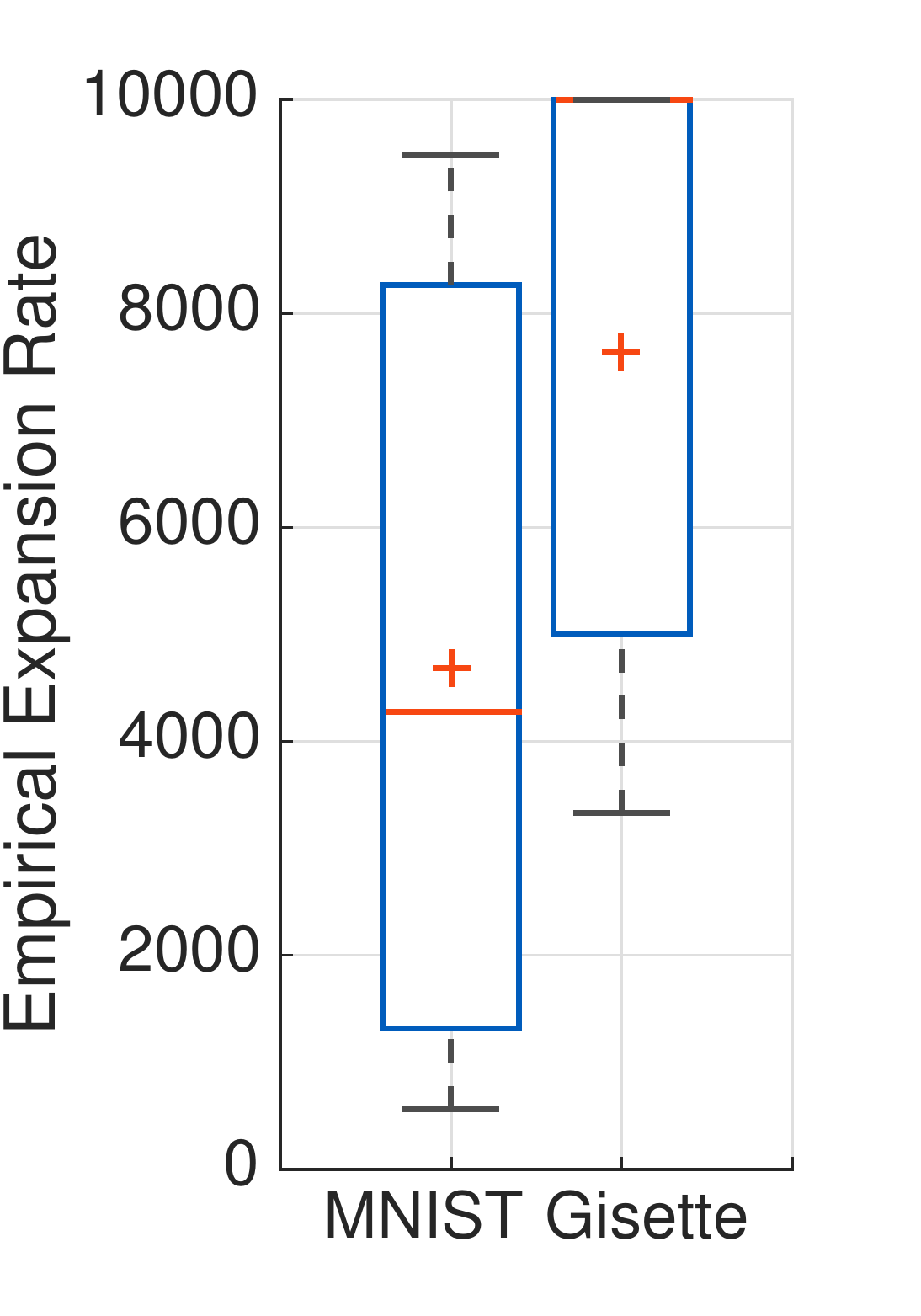}
        \label{fig:y equals x}
    \end{subfigure} 
    \begin{subfigure}[b]{0.27\textwidth}
        \centering
        \includegraphics[width=\textwidth]{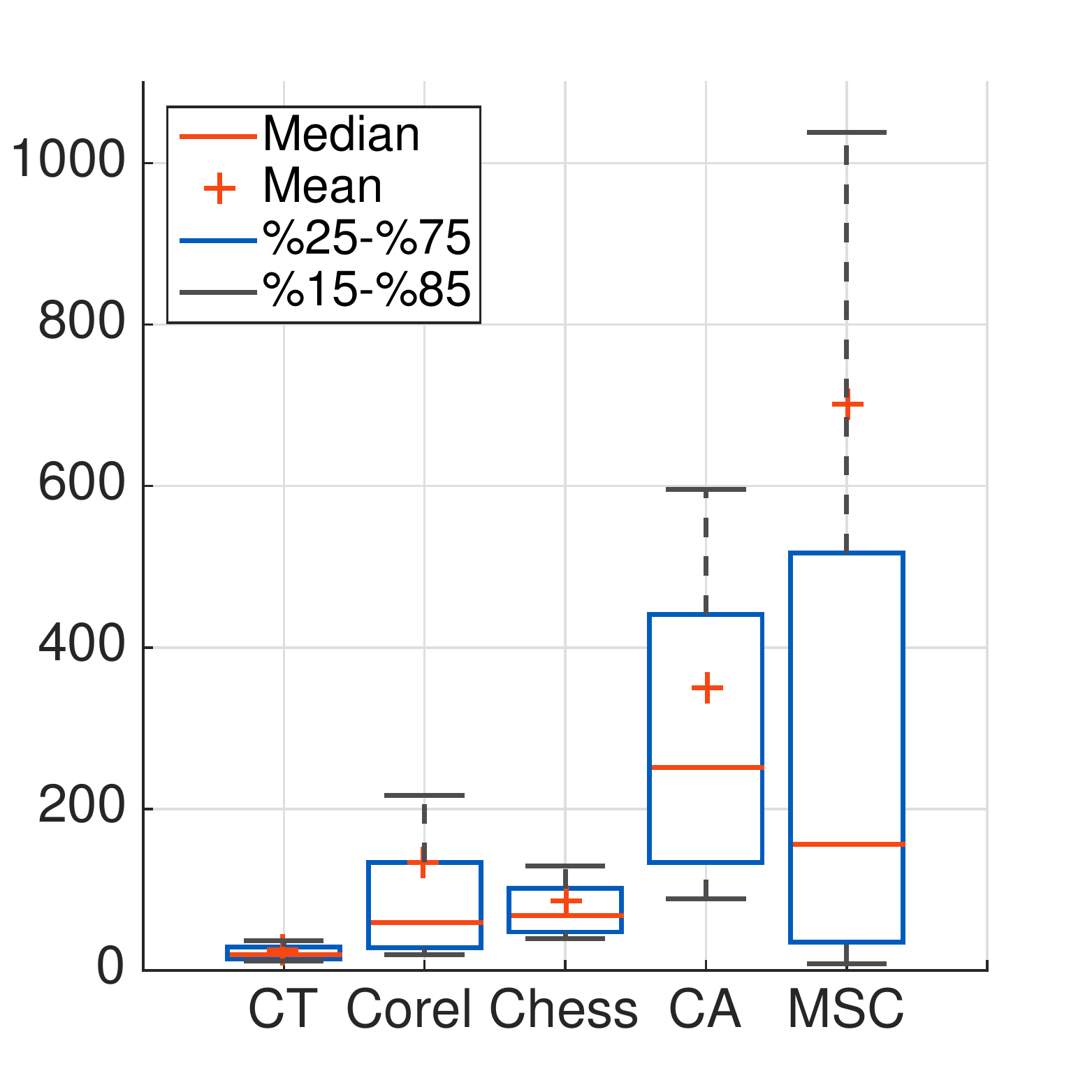}
        \label{fig:three sin x}
    \end{subfigure}
    \vspace{-.5cm}\caption{Distribution of empirical expansion rates $\tilde{c}(x)$ estimated for various datasets. Each bar represents the distribution of pointwise expansion rates for the corresponding dataset. CoAuth and CoverType datasets are abbreviated as ``CA" and ``CT" respectively. Note that the range of the expansion rates is 10 times higher for the first two datasets, thus we plotted them separately.}
    \label{fig:ExpRatesAll}
\end{figure}

For our theoretical analysis, we used the smallest possible $\tilde{c}$ for the whole dataset. However the distribution of pointwise values $\tilde c(x)$ is a more practical criterion to consider. Except for MNIST and Gisette, these values are reasonably small. Therefore, the values can justify the validity of the assumption on real datasets.
\section{CONCLUSIONS}

Comparison-based nearest neighbor search is a fundamental ingredient in machine learning algorithms in the comparison-based setting. Because triplet comparisons are expensive, we investigate the query complexity of comparison-based nearest neighbor algorithms. In particular, we study the comparison tree, which leads to a nice and simple, yet adaptive data structure. We prove that under strong conditions on the underlying metric, the comparison tree has logarithmic height, and we can bound the error of the nearest neighbor search. We also show in simulations that comparison trees perform not much worse than Euclidean data structures (albeit using much less information about the data), and perform favorably to other comparison-based methods if we take both the number of triplet comparisons and the nearest neighbor errors into account. 

There are still a number of interesting open questions to address. The conditions we use in our analysis are rather strong, and this seems to be the case for all other papers in this area as well. Can they be considerably weakened? Can we prove that our conditions will be satisfied with small constants if we sample point from a nice metric or Euclidean space? Finally, all the above work assumes that a ground truth for the triplet comparisons exists and that the answers to the triplet queries are always correct. It would be interesting to see how the query complexity of the comparison tree increases if the error in the triplets increases. 

\subsubsection*{Acknowledgements}
This work is supported by DFG (SFB 936/project Z3), Research Unit 1735, and the Institutional Strategy of the University of Tübingen (ZUK 63).

\clearpage
\bibliographystyle{plainnat}

\end{document}